%% file: main.tex
\pdfoutput=1 
\documentclass[11pt]{article}
\usepackage[font=libertinus, citestyle=numeric]{kurbanlab}
\input{affiliations}

\usepackage{colortbl}   
\usepackage{pifont}     

\definecolor{ink}{HTML}{1A2530}
\definecolor{natblue}{HTML}{0072B2}        
\definecolor{natvermillion}{HTML}{D55E00}  
\definecolor{natgreen}{HTML}{009E73}       
\definecolor{natorange}{HTML}{E69F00}      
\definecolor{natgray}{HTML}{6B7280}
\definecolor{panelblue}{HTML}{EAF3F9}
\definecolor{panelgray}{HTML}{F2F4F6}

\newcommand{\cyes}{\textcolor{natgreen}{\ding{51}}}
\newcommand{\cno}{\textcolor{natgray!65}{\ding{55}}}
\newcommand{\costhi}{\textcolor{natvermillion}{high}}
\newcommand{\costmid}{\textcolor{natorange!85!black}{medium}}
\newcommand{\costlo}{\textcolor{natgreen}{low}}
\newcommand{\headrow}{\rowcolor{panelgray}}
\newcommand{\ourrow}{\rowcolor{panelblue}}

\newcolumntype{Y}{>{\raggedright\arraybackslash}X}

\newcommand{\feat}{\phi}
\newcommand{\sAUC}{\mathrm{sAUC}}

\title{Detect Early, Escalate Rarely: Anytime Detection of AI-Generated Video
from the Compressed Bitstream}
\Subtitle{Streaming detection from the codec motion field, with an anytime guarantee
and a priced compute frontier}
\RunningTitle{Anytime AIGV detection from the compressed bitstream}

\Author{Mert Onur Cakiroglu}{iub}
\Author{Mehmet Dalkilic}{iub}
\Author[corresponding=hkurban@hbku.edu.qa]{Hasan Kurban}{hbku}

\Keywords{AI-generated video detection; streaming perception; compressed-domain video;
anytime-valid inference; cascade deferral; efficient inference}
\CodeURL{https://github.com/KurbanIntelligenceLab/streamdet}
\Venue{}

\begin{document}
\maketitle

\begin{abstract}
Detectors for AI-generated video are evaluated offline. A clip is decoded to pixels and scored
once, increasingly by a large vision-language model. Detection, however, is deployed online. We
recast the task as streaming perception and score the motion field the codec already wrote into
the bitstream. Reading that field is a parse, not a pixel-domain forward pass. Because the running
aggregate is monotone, one end-calibrated threshold is anytime-valid at the data-dependent
decision time. Recalibrating at each prefix is not. Escalation is priced in closed form. A compute
budget maps to a deferral window, on a frontier monotone exactly where the deferral condition
holds. On matched GenVidBench the codec stage reaches full-length AUC $0.64$ at five orders of
magnitude less compute than a pixel CNN, on CPU. Its gate holds the stopping-time false-positive
rate at target while the real data match its calibration, and drifts above it under distribution
shift. Deferring $15\%$ of clips lifts accuracy from $0.75$ to $0.78$ at $7\times$ less compute
(paired: McNemar $p{<}10^{-6}$). The stage-1 ordering replicates on AIGVDBench. We introduce no new
detector. The contribution is the reframing, two guarantees, and the measured frontiers.
Code, configurations, and evaluation splits:
\url{https://github.com/KurbanIntelligenceLab/streamdet}.
\end{abstract}

\printkeywords

\section{Introduction}
\label{sec:intro}

\begin{figure}[t!]
\centering
\kilgraphics[0.70\textwidth]{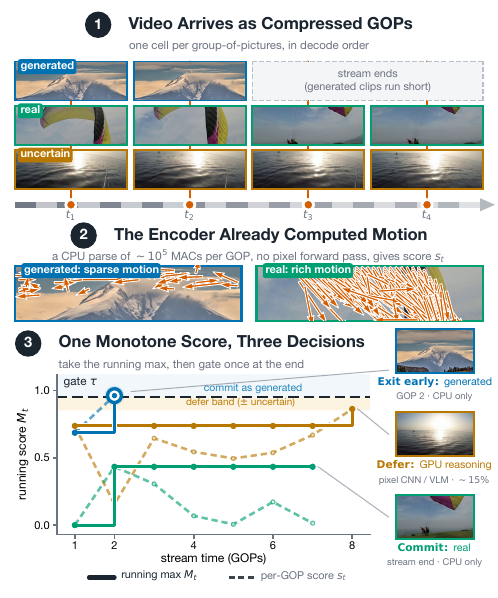}
\caption{\textbf{Detect early, escalate rarely.} A worked example on three clips from
the matched cell; every element shown is measured, not drawn.
\textbf{(1)}~Video arrives one group-of-pictures (GOP) at a time; generated clips also end early (the length asymmetry analyzed in E5). \textbf{(2)}~The codec has already written a quantized motion
field; a CPU parse reads it for about $10^5$ multiply-accumulate operations (MACs), and its texture separates sparse generated
from rich real motion. \textbf{(3)}~Per-GOP scores $s_t$ (dashed) aggregate into a running maximum $M_t$ (solid), compared once against the end-calibrated threshold $\tau$ (Proposition~\ref{prop:avi}). The generated clip exits at GOP~2; the real clip commits at stream end. A clip finishing in the uncertain band escalates to a GPU stage (about $15\%$ of inputs).}
\label{fig:teaser}
\end{figure}

Text-to-Video systems such as Sora~2, Veo~3, and Kling now produce AI-generated video (AIGV)
so realistic that detection has become a practical necessity. Such detection is
typically benchmarked offline but deployed online; the two settings, however, reward different types of
detectors. An offline detector receives the entire clip at once, decodes it into pixels, and
produces a single score. Recent methods increasingly use vision-language models (VLMs) to reason
over complete clips \cite{vidguardr1,skyra,ivyfake,busterx,videoveritas}. In contrast, deployment
video arrives sequentially, and the detector must make a decision within a latency budget. At
deployment scale, running a multi-billion-parameter model on every clip is computationally
impractical. Recent benchmarks also show  that strong
multimodal models perform near chance on challenging, in-the-wild AIGV
\cite{artifactbench,genvideolens}, making per-clip scoring with such models not only expensive but also, more importantly, unreliable.

Several neighboring lines of work address parts of this problem, but none combines streaming
operation, compressed-domain analysis, and compute-adaptive detection for fully-synthetic AIGV.
Streaming perception
folds latency into accuracy, but for object detection and action recognition, not generated video
\cite{streamyolo,edgeoar}. Real-time deepfake detectors operate online but analyze decoded facial imagery
\cite{activeprobe,realtimedf}, while motion- and frequency-based methods remain offline and
face-specific \cite{opticalflowdf,compdeepfake}, including one that reads H.264 motion-vector (MV)
fields \cite{mvdeepfake}. Detectors for fully-synthetic AIGV decode pixels and
run a heavy per-clip forward pass, forensic and VLM-reasoning alike
\cite{waverep,d3,vidguardr1,nativescale}. The adaptive-computation toolkit for spending compute
per input is mature \cite{violajones,branchynet,earlyexitcascade,cascadedefer}, and concurrent
work brings it to AIGV. CoCoDetect gates a 3D convolutional neural network (CNN) and escalates
uncertain clips to a multimodal large language model (MLLM), but current AIGV systems remain
offline and pixel-domain \cite{cocodetect}.

The compressed bitstream already contains a signal suitable for streaming detection. For each group of pictures (GOP), the codec writes a quantized motion field into the stream.
Extracting and scoring this field requires only lightweight CPU-side processing and no
pixel-domain neural-network forward pass \cite{vidaudit}. In addition, the signal is available at every prefix, allowing the detector to act before the
complete clip has arrived. Pixel and VLM detectors must instead decode frames and run a forward
pass, so their natural unit is a whole clip rather than an arriving chunk, and their resizing and
re-encoding can weaken the high-frequency forensic cues some of them rely on
\cite{nativescale}. We aggregate the per-GOP scores monotonically and apply a
single end-calibrated threshold, yielding false-positive control at the detector's
data-dependent stopping time. Clips that remain near the decision boundary are deferred to a
heavier pixel or VLM detector (Figure~\ref{fig:teaser}). On matched GenVidBench, the codec stage
reaches AUC $0.80$ from the first GOP alone, at five orders of magnitude less compute than a pixel
detector. Deferring $15\%$ of clips raises accuracy from $0.75$ to $0.78$ while using seven times
less compute than pixel-domain scoring on every clip.

We make three contributions. \textbf{Streaming AIGV detection: a problem, a protocol, and
metrics.} We recast detection as streaming perception and define anytime AUC (area under the ROC
curve) over the observed prefix, a latency-budgeted streaming AUC, and recall at a fixed
false-positive rate versus prefix, over public benchmarks with a matched harness
(Sections~\ref{sec:problem} and~\ref{sec:exp}). \textbf{An anytime compressed-domain front-end
with a temporal guarantee.} We aggregate GOP-level codec scores into a running maximum gated by
one threshold. Monotonicity buys the guarantee (Propositions~\ref{prop:anytime}
and~\ref{prop:avi}). A threshold calibrated once, at stream end, controls the false-positive rate
at the gate's own firing time, which the data choose (under the calibration null). Per-prefix
recalibration does not (Sections~\ref{sec:method} and~\ref{sec:theory}). \textbf{Reasoning on
demand, with a priced frontier.} We escalate only the clips the gate calls uncertain, to an
expensive pixel or VLM detector. Expected compute then has a closed form, so a budget maps to a
deferral window and one point on the compute-accuracy frontier. The frontier is monotone only
where the deferral condition holds. Stage~2 must be at least as accurate as stage~1 there
(Proposition~\ref{prop:cascade}, Corollary~\ref{cor:budget}; Section~\ref{sec:exp} measures both
frontiers).

\section{Related Work}
\label{sec:related}

\paragraph{AI-generated video detection} Non-reasoning detectors span frequency and forensic
features \cite{waverep}, representation-geometry cues \cite{restrav}, training-free temporal
statistics \cite{d3}, and spectral-semantic fusion \cite{specsemnet}. Large benchmarks
evaluate them \cite{genvidbench,aigvdbench,demamba}. The most active recent thread couples detection with
explanation through VLMs and reinforcement learning
\cite{vidguardr1,skyra,ivyfake,busterx,videoveritas}. These are accurate and interpretable but
heavy and offline, remaining clip-level even at native resolution \cite{nativescale}. Reliability is the other issue. Off-the-shelf multimodal models are often near-random on hard
or in-the-wild AIGV \cite{genvideolens,artifactbench}. We do not propose
another such detector. We use one as an on-demand escalation stage and ask how rarely it must be
called. Our front-end faces one known threat. A recent audit shows that motion-cue detectors can
exploit dataset motion biases (``generated videos move less'') rather than generation artifacts
\cite{motionbias}. Our stage~1 reads codec motion vectors, so we adopt the concern directly and evaluate under motion-bias control (E7 in Section~\ref{sec:exp}). We also report the identity floor the matched harness gets from scoring real clips against real ones \cite{vidaudit}.

\paragraph{Streaming perception} Streaming perception jointly evaluates
latency and accuracy for online object detection \cite{streamyolo}, and online action
recognition exits early on edge devices \cite{edgeoar}. Streaming VLM inference is
emerging \cite{codecsight}. We adopt this setting for detection. A decision is a function of
the observed prefix and a latency budget, not of the whole clip.

\paragraph{Real-time and prior-cue deepfake detection} The closest detectors come from the
deepfake (face-manipulation) literature. Real-time methods authenticate faces in calls through active
probing \cite{activeprobe} or spatial-frequency artifacts \cite{realtimedf}. These run online but
read faces from decoded pixels. Offline methods use optical flow \cite{opticalflowdf} or
compression-aware frequency cues \cite{compdeepfake}, at the clip level, again for faces. The nearest prior reads the same H.264 motion vectors as our stage~1, confirming that they carry a cheap temporal cue \cite{mvdeepfake}. But it crops faces from decoded RGB, classifies face-swap clips offline, and raises streaming only as a prospect. We differ on every axis at once: fully-synthetic AIGV, the bitstream
as an always-on streaming front end, and an anytime, compute-adaptive decision with guarantees.
Table~\ref{tab:positioning} compares the closest work on four axes: online
operation, pixel inference, adaptive compute, and cost per clip.

\paragraph{Adaptive computation} Spending computation per input is classical, from boosted cascades \cite{violajones} to early-exit networks \cite{branchynet}. It continues in window-based early-exit ensembles that route only near-boundary inputs deeper \cite{earlyexitcascade}, and in confidence-based cascade deferral with accuracy conditions \cite{cascadedefer}. The current test-time-compute literature uses the same idea. An expensive model is invoked only when a
cheap model is uncertain \cite{ttcsurvey,bestroute}. Concurrent work brings this pattern to AIGV. CoCoDetect gates a pixel-domain 3D CNN and routes uncertain clips to an MLLM \cite{cocodetect}. But it is offline and clip-level, decodes every frame, and offers no anytime or budget guarantees. Our cascade applies these established methods. The novelty is the streaming
problem, the bitstream-native anytime stage, and the resulting frontiers and guarantees.

\paragraph{Compressed-domain video and provenance} Beyond detection, codec signals drive
object tracking without RGB decoding \cite{seewithoutdecode} and efficient streaming VLM
inference \cite{codecsight}. Our own VidAudit toolkit uses them as an offline forensic substrate
\cite{vidaudit}. None detect generated video as it streams. Source-side provenance labels \cite{c2pa,synthid,euaiact} are complementary, but require source cooperation. Transcoding can strip or degrade them, so passive detection remains necessary. The codec signal we read is produced by that same re-compression (discussion in \cref{app:related}).

\begin{table}[t]
\centering
\small
\setlength{\tabcolsep}{5pt}
\renewcommand{\arraystretch}{1.20}
\begin{tabularx}{\textwidth}{@{}Ycccc@{}}
\toprule
\headrow
Approach & online & \begin{tabular}[c]{@{}c@{}}no pixel\\inference\end{tabular} &
\begin{tabular}[c]{@{}c@{}}adaptive\\compute\end{tabular} &
\begin{tabular}[c]{@{}c@{}}cost /\\clip\end{tabular} \\
\midrule
VLM-reasoning detector~\cite{vidguardr1,skyra} & \cno & \cno & \cno & \costhi \\
Forensic / spectral~\cite{waverep,d3} & \cno & \cno & \cno & \costmid \\
Offline benchmarks~\cite{genvidbench,aigvdbench} & \cno & \cno & \cno & \textcolor{natgray}{n/a} \\
Real-time deepfake, face~\cite{activeprobe,realtimedf} & \cyes & \cno & \cno & \costmid \\
Flow/MV/freq.\ deepfake~\cite{opticalflowdf,mvdeepfake,compdeepfake} & \cno & \cno & \cno & \costmid \\
Streaming perception~\cite{streamyolo} & \cyes & \cno & partial & \costmid \\
Online action recognition~\cite{edgeoar} & \cyes & \cno & early-exit & \costlo \\
Gated pixel cascade~\cite{cocodetect} & \cno & \cno & defer & \costmid \\
\ourrow
\textbf{This paper} & \cyes & \cyes$^{*}$ & \textbf{defer} & \textcolor{natgreen}{\textbf{low}}$^{+}$ \\
\bottomrule
\end{tabularx}
\caption{Positioning against the surrounding lines of work
(\cyes/\cno${=}$has/lacks; cost \costlo/\costmid/\costhi). ``No pixel inference'': the always-on stage scores the codec motion field, not decoded frames. Obtaining the field is a parse by design. $^{*}$the always-on stage~1. $^{+}$the reference implementation's caveat is under Cost.}
\label{tab:positioning}
\end{table}

\section{Streaming AIGV Detection}
\label{sec:problem}

\paragraph{Stream model} A video arrives as an ordered sequence of chunks $g_1,g_2,\dots,g_N$. At
prefix length $t$ the detector has observed $g_{1:t}$ and must be able to output a decision
$D_t\in\{\text{real},\text{generated},\text{wait}\}$ from it alone. Each chunk carries its
arrival/processing latency, so a decision at prefix $t$ has a well-defined latency $\ell_t$. The chunk
is the arrival and decision unit, a fixed window of $16$ frames by default. Its size is a
design knob we ablate (E7). Under the harness's
canonical H.264 re-encode (closed GOP of $16$; Section~\ref{sec:exp}) a chunk is one encoder GOP, so the two terms are interchangeable below. In the wild a clip's GOP structure can be far coarser, so the detector windows arriving
frames itself (detail in the appendix).

\paragraph{Compressed-domain features} The encoder has already summarized each GOP's motion: a
quantized motion-vector field, written into the bitstream. For each GOP we read that field and
condense it into a fixed feature $\feat_t\in\mathbb{R}^{13}$, four global motion statistics and nine spectral dimensions, following our VidAudit toolkit \cite{vidaudit}. The scorer therefore runs no pixel-domain forward pass (the implementation
caveat is under Cost). A base scorer $f_R$ maps a feature (or the running
aggregate) to a score $s_t\in\mathbb{R}$, higher meaning ``generated.''

Streaming evaluation needs metrics indexed by time. Definition~\ref{def:metrics} makes the prefix and the decision latency explicit.

\begin{definition}[Anytime accuracy and streaming AUC]
\label{def:metrics}
Let $M_t$ be the detector's running score after prefix $t$. The \emph{anytime AUC} at prefix $t$,
$\mathrm{AUC}(t)$, is the ROC-AUC of $M_t$ over the dataset. For a latency budget $B$ (a maximum
admissible prefix or wall-clock), the \emph{latency-budgeted streaming AUC} is
$\sAUC(B)=\mathrm{AUC}(t_B)$ where $t_B$ is the largest prefix reachable within $B$; the full
\emph{anytime curve} is $\{(\ell_t,\mathrm{AUC}(t))\}_t$. We also define recall at a fixed
false-positive rate as a function of prefix, and the decision-latency distribution under a
confidence gate.
\end{definition}

Computed on public benchmarks with a matched harness, these metrics are the artifact we release. Any detector can be plotted on the same axes; an offline one enters as a single point at the full prefix.

\section{Method}
\label{sec:method}

\paragraph{Anytime front-end} The always-on stage maintains a running aggregate of the per-GOP
codec scores. We use a sequential disjunctive aggregate $M_t=\max_{i\le t}s_i$, and a
\emph{single} decision threshold $\tau$ calibrated to the final-prefix score distribution on real clips (the calibration null). A confidence gate raises ``generated'' as soon as
$M_t\ge\tau$. By
Proposition~\ref{prop:avi} this one threshold controls the false-positive rate at every prefix
and at the data-dependent decision time. If the stream ends or the latency budget expires
without crossing $\tau$, we ask how close $M_t$ came to $\tau$. We fix a
\emph{deferral width} $w$. The clip is ``real'' when $M_t<\tau-w$, and uncertain when
$M_t\in[\tau-w,\tau)$, with $t$ the decision point.

\paragraph{Reasoning on demand} A clip whose final aggregate lands in the one-sided band
$W=[\tau-w,\tau)$ is escalated to an expensive stage-2 detector $h$ applied to the observed
prefix. By default $h$ is a lightweight pixel CNN, optionally a small VLM reasoning detector
\cite{ivyfake,vidguardr1,skyra}. The width $w$ is the single knob setting the deferral rate. The final
decision is stage~2's verdict on deferred clips and stage~1's elsewhere. Section~\ref{sec:theory} analyzes this cascade with a generic score $s$ and window $W$. Here $s$ is the aggregate $M$ at the decision point and $W$ is this band. Sweeping $w$ is thus exactly the budget sweep of Corollary~\ref{cor:budget}. The full procedure is Algorithm~\ref{alg:stream}.

\begin{algorithm}[t]
\caption{Streaming, compute-adaptive AIGV detection (one clip; $\tau$ and $w$ are set once,
offline).}
\label{alg:stream}
\begin{algorithmic}[1]
\Require GOP stream $g_1,\dots,g_N$; base scorer $f_R$; decision threshold $\tau$ (calibrated per
Proposition~\ref{prop:avi}); deferral width $w$ (band $W=[\tau-w,\tau)$); stage-2 detector $h$;
latency budget $B$
\Ensure decision and decision latency
\State $M_0 \gets -\infty$
\For{$t=1,2,\dots,N$ \textbf{while} latency $\ell_t\le B$}
  \State $\feat_t \gets \mathrm{ParseMV}(g_t)$ \Comment{no pixel-domain forward pass}
  \State $s_t \gets f_R(\feat_t)$; \quad $M_t \gets \max(M_{t-1},s_t)$
  \If{$M_t \ge \tau$}
    \State \Return (\textbf{generated}, $\ell_t$) \Comment{confident early exit}
  \EndIf
\EndFor
\State \textit{// decision point reached without crossing $\tau$}
\If{$M_t < \tau - w$}
  \State \Return (\textbf{real}, $\ell_t$) \Comment{confident}
\Else
  \State $y \gets h(\mathrm{Decode}(g_{1:t}))$ \Comment{uncertain ($M_t\in W$): stage 2 only here}
  \State \Return ($y$, $\ell_t$)
\EndIf
\end{algorithmic}
\end{algorithm}

\paragraph{Cost} Stage-1 scoring costs about $10^{5}$ MACs per GOP on CPU, five
orders of magnitude below a pixel-CNN forward pass, with no GPU. The only GPU cost is stage~2, on
the deferred fraction. One caveat concerns wall-clock time. Our reference implementation reads the field through the decoder's side-data export, and so pays a frame decode. A purpose-built parser avoids that decode, and a platform already performs the decode anyway. Either way, the compute the codec stage adds is the $10^{5}$-MAC parse, the figure we
report (full accounting in the appendix).

\section{Analysis}
\label{sec:theory}

We fix a perturbation-free streaming setting and analyze the two design choices. The
aggregator governs when the detector may speak, and the deferral rule governs what speaking
costs. Proof ideas are given here, and the two longer proofs are in the appendix. The
first result prices waiting.

\begin{proposition}[Anytime monotonicity and its false-positive cost]
\label{prop:anytime}
Under the disjunctive aggregate $M_t=\max_{i\le t}s_i$ with a \emph{fixed} threshold $\tau$, for
every clip the events $\{M_t\ge\tau\}$ are nested (non-decreasing) in $t$. Consequently the
population true-positive rate $\mathrm{TPR}_\tau(t)=\Pr_{\mathrm{gen}}[M_t\ge\tau]$ and
false-positive rate $\mathrm{FPR}_\tau(t)=\Pr_{\mathrm{real}}[M_t\ge\tau]$ are both non-decreasing
in $t$, where for a clip of length $N<t$ we set $M_t:=M_N$ (the maximum freezes at stream end), so
the rates are over the full population at every $t$.
\end{proposition}

\begin{proof}
$M_t=\max(M_{t-1},s_t)\ge M_{t-1}$, so $M_t$ is non-decreasing in $t$ pointwise; hence
$\{M_{t-1}\ge\tau\}\subseteq\{M_t\ge\tau\}$ as events. Taking probabilities under the generated
and real distributions gives the two monotonicities.
\end{proof}

Proposition~\ref{prop:anytime} makes the latency-accuracy trade-off precise. Observing more
can only raise recall and the false-positive rate, because a disjunctive rule
accumulates chances to fire. Waiting can cost false alarms.

The false-positive side needs more care. A streaming detector emits its verdict at a
\emph{data-dependent stopping time}, the first prefix at which the gate is confident. A
fixed-prefix guarantee need not survive that stopping rule. The sequential-analysis literature calls this the optional-stopping problem and answers it with
anytime-valid inference \cite{wald1945,savi,evalues,inducedseqtest}: error control at every
prefix and under any stopping rule. For our monotone aggregate, one end-calibrated threshold
supplies that control at a target false-positive level $\alpha$. Per-prefix recalibration removes it.

\begin{proposition}[Anytime-valid false-positive control]
\label{prop:avi}
Let $P_0$ be the real (null) distribution, let $M_t=\max_{i\le t}s_i$, and let the detector raise
``generated'' at the alarm time $\sigma=\inf\{t\le N: M_t\ge\tau\}$ (and ``real'' if no such $t$).
\begin{enumerate}\itemsep1pt
\item[(i)] \textbf{A single end-calibrated threshold is anytime-valid.} Calibrate the threshold
$\tau$ so that the null tail at the final prefix is controlled: $P_0(M_N\ge\tau)\le\alpha$. Such a
$\tau$ exists: the upper $(1-\alpha)$-quantile of $M_N$ when $M_N$ has no atom there (then with
equality), and in general any $\tau$ above it. Then the false-positive rate at the data-dependent alarm time equals
the final-prefix rate and is controlled, $P_0(\sigma\le N)=P_0(M_N\ge\tau)\le\alpha$, and
simultaneously $P_0(M_t\ge\tau)\le\alpha$ for every prefix $t\le N$.
\item[(ii)] \textbf{Per-prefix recalibration is not.} Choosing $\tau_t$ so that
$P_0(M_t\ge\tau_t)=\alpha$ at each prefix (an atomless null, for concreteness) controls the
per-prefix rate but not the alarm-time rate: $P_0(\exists\,t\le N: M_t\ge\tau_t)\ge\alpha$, with
strict inequality iff an early crossing that the final test misses has positive probability,
$P_0(\exists\,t<N:M_t\ge\tau_t,\,M_N<\tau_N)>0$: the multiple-looks inflation.
\end{enumerate}
\end{proposition}

\emph{Proof idea.} Monotonicity collapses the union over prefixes into a single event,
$\{\exists\,u\le t: M_u\ge\tau\}=\{M_t\ge\tau\}$. The alarm therefore fires at all iff it fires at the end. Part (i) is then the calibration of $\tau$ read back. For (ii), the per-prefix thresholds $\tau_t$ are non-decreasing (a quantile of the stochastically increasing $M_t$). So $M$ can cross an early boundary yet finish below the higher final one. That is exactly the mass the union carries beyond $\alpha$.
\Cref{app:proofs} proves both parts, the second via a lemma on the monotonicity of $M_t$.

\smallskip
\noindent Proposition~\ref{prop:avi} thus guarantees level $\alpha$ at the detector's actual
operating point, the stopping time, not merely at a fixed prefix. With $\tau$
so calibrated, Definition~\ref{def:metrics}'s anytime curve traces recall against latency at a
controlled false-alarm level.

The deferral rule determines what escalation costs and when it is worth
paying.

\begin{proposition}[Cascade compute and accuracy]
\label{prop:cascade}
Let stage~1 produce a score $s$ with per-GOP compute $C_1$, and let stage~2 have compute $C_2$ and
be invoked exactly when $s\in W$, both at constant unit costs; the cascade decides by stage~2 on
$\{s\in W\}$ and by stage~1's deterministic rule elsewhere. Write $\mathrm{err}_i^{A}$ for stage~$i$'s \emph{error mass} on a region
$A$: the probability of being wrong \emph{and} in $A$. Let $\bar t$ be the expected number of GOPs
processed, and $\mathrm{err}_1$ stage~1's error if it decided everywhere. The expected per-clip
compute and the cascade error are then
\[
\begin{aligned}
\mathbb{E}[C]&=\bar{t}\,C_1+\Pr(s\in W)\,C_2,\\[2pt]
\mathrm{err}_{\mathrm{casc}}&=\mathrm{err}_1-\bigl(\mathrm{err}_1^{W}-\mathrm{err}_2^{W}\bigr).
\end{aligned}
\]
The bracketed term is the \emph{deferral gain} $\Delta(W):=\mathrm{err}_1^{W}-\mathrm{err}_2^{W}$,
so the cascade differs from stage~1 by exactly that gain. Hence $\mathrm{err}_{\mathrm{casc}}\le\mathrm{err}_1$
iff $\mathrm{err}_2^{W}\le\mathrm{err}_1^{W}$, the confidence-deferral condition, and strictly when
stage~2 is strictly better on the deferred set.
\end{proposition}

\emph{Proof idea.} Compute is linear over the GOPs processed and the deferral event. The
cascade is stage~1 off $W$ and stage~2 on $W$, so subtracting the identity
$\mathrm{err}_1=\mathrm{err}_1^{\bar W}+\mathrm{err}_1^{W}$, with $\bar W$ the complement of $W$, leaves exactly the deferral gain. The
appendix carries this out.

\begin{corollary}[Budgeted compute-accuracy frontier]
\label{cor:budget}
For an expected-compute budget $\mathbb{E}[C]=b$ with
$\bar t\,C_1\le b\le \bar t\,C_1+\Pr(s<\tau)\,C_2$ (gate-fired clips are never deferred), a width
$w$ with deferral probability $\Pr(s\in W)=(b-\bar t\,C_1)/C_2$ meets the budget exactly when the
score law is atomless below $\tau$; the sweep leaves $\bar t$ and $\mathrm{err}_1$ unchanged, since
deferral is decided only at the decision point. The achieved operating point is
\[
\bigl(\mathbb{E}[C],\,\mathrm{acc}\bigr)=\bigl(\bar t\,C_1+\Pr(s{\in}W)\,C_2,\;
1-\mathrm{err}_1+\Delta(W)\bigr),
\]
with $\Delta(W)=\mathrm{err}_1^{W}-\mathrm{err}_2^{W}$ the deferral gain of
Proposition~\ref{prop:cascade}. Sweeping the deferral width $w$ (hence $W=[\tau-w,\tau)$) traces
this frontier. It is monotone, accuracy non-decreasing in the budget, iff every marginal region
added as $W$ widens satisfies the deferral condition ($\mathrm{err}_2\le\mathrm{err}_1$ there).
Where the condition fails, the frontier bends back, because escalation is being spent where stage~1 was
already correct.
\end{corollary}

The condition in Proposition~\ref{prop:cascade} is the standard confidence-deferral condition \cite{cascadedefer}. Corollary~\ref{cor:budget}'s monotonicity is its marginal form, since each widening adds exactly the deferral gain on the newly deferred region (proof in \cref{app:proofs}). Nothing forces stage~2 to be better on the deferred set, so we
test the condition directly (E3, E4). Together the two guarantees turn ``how early can we decide'' into a calibrated anytime curve. They turn ``how rarely must we call the expensive detector'' into a budgeted compute-accuracy frontier.

\section{Experiments}
\label{sec:exp}

This section specifies the protocol and reports the runs. We run the cheap sanity control first
(E1), then the headline latency-accuracy frontier (E2), the compute-accuracy frontier and deferral
guarantee (E3--E4), earliness and stopping-time control (E5), a cross-dataset replication (E6), and
ablations (E7).

\paragraph{Data, harness, and statistics} We use the matched GenVidBench subset as the
comparison set, a ``cell'' of ${\sim}27$k clips matched between real and
generated sources. AIGVDBench serves for cross-dataset replication
\cite{genvidbench,aigvdbench}, with the harness, splits, and seeds of our VidAudit toolkit \cite{vidaudit}. Clips enter through that harness's
canonical H.264 re-encode (closed GOP of $16$), which normalizes encoder fingerprints. E6
deliberately departs from it. We simulate streaming by presenting each clip GOP-by-GOP in decode
order. The two stages are those of Section~\ref{sec:method}. Stage~1 is the CPU-only codec-MV detector. Stage~2 (default) is a lightweight pixel model (CLIP ViT-B/32) \cite{clip}. A small VLM (Ivy-xDetector, Qwen2.5-VL-3B \cite{ivyfake}) is the escalation upgrade (E4). A clip whose chunks are all low-motion receives a floor score and abstains. It is not dropped, so the evaluated population stays the audited cell. The workload fits one $12$\,GB consumer GPU (hardware: \cref{app:compute}).
Principal numbers are leave-one-generator-out means over the
seven held-out generators (seed~42), with percentile-bootstrap 95\% confidence intervals (CIs).
Decision-accuracy points carry clip-level bootstrap CIs ($10{,}000$ resamples within folds).
Every deferral gain carries a paired test (paired bootstrap and exact McNemar). Multi-seed
hardening is left to future work. Every number except the VLM and ReStraV MAC estimates comes from the
released code path. That path also checks Propositions~\ref{prop:anytime}--\ref{prop:cascade}
and Corollary~\ref{cor:budget} by assertion on synthetic data. \Cref{app:pergen} tabulates
per-generator spread.

\paragraph{Baselines} We compare four systems: (a)~the codec stage-1 alone (anytime); (b)~a pixel
CNN scored on every processed chunk; (c)~a VLM detector scored once per clip, the high-cost
reasoning baseline and an upper bound on cost; and (d)~clip-level offline versions of (a)--(c) for
reference. Each is reported on the metrics of Definition~\ref{def:metrics}, plus expected compute (MACs
and ms) and the deferral rate. ReStraV \cite{restrav}, re-trained under the identical
protocol, anchors the accuracy-forward end (Table~\ref{tab:results}; Figure~\ref{fig:pareto}). Real-time face-deepfake detectors \cite{activeprobe,realtimedf} are complementary, not comparison
baselines (appendix).
A second current external anchor makes the same point. VideoVeritas (Qwen3-VL-8B, zero-shot, run
under the audited protocol) scores the balanced $2{,}127$-clip subsample (E4) at balanced accuracy
$0.82$ $[0.80,0.83]$ and AUC $0.83$ $[0.82,0.85]$, at ${\sim}10^{13}$ MACs per clip. Its
per-generator spread is bimodal. It is near ceiling on cogvideo, pika and vc2 ($0.92$--$0.95$) but
below chance on svd, musev and ms ($0.15$--$0.45$), where it confidently calls generated clips
real. Heavy reasoning detectors are therefore accurate on average and unreliable exactly where
detection is hardest, which is the regime the cascade rations them for.

\begin{figure}[t]
\centering
\kilgraphics[0.72\textwidth]{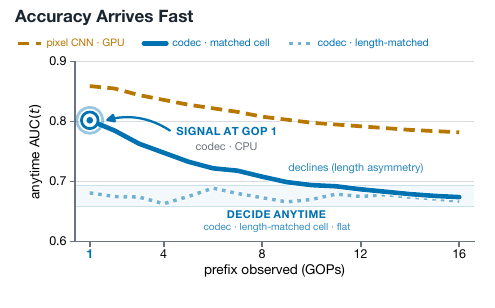}
\caption{\textbf{The latency frontier} (matched 27k cell, leave-one-generator-out; E2, E5). The pixel CNN (orange) leads throughout, but the codec stage (blue) is informative from the first GOP ($\sAUC(1)=0.80$, CPU-only). Its decline reflects the cell's length asymmetry. On a length-matched cell the curve is flat (dotted; band), so the detector can decide at any prefix.}
\label{fig:anytime}
\end{figure}

\begin{figure}[t]
\centering
\kilgraphics[0.72\textwidth]{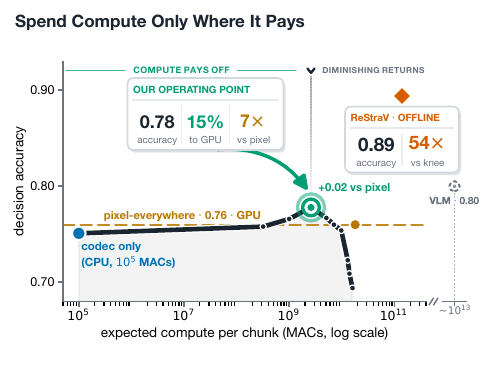}
\caption{\textbf{The compute frontier} (same cell; E3). The cascade frontier peaks at the knee ($0.78$ at $15\%$ deferral). That is above pixel-everywhere $0.76$, at $7\times$ less compute. The frontier then bends back. ReStraV (offline) reaches $0.89$ at ${\sim}54\times$ the knee's per-chunk compute. The
VLM sits off-scale ($\sim\!10^{13}$~MACs per clip; $0.80$ on the subsample).}
\label{fig:pareto}
\end{figure}

\begin{table}[t]
\centering
\small
\setlength{\tabcolsep}{5pt}
\renewcommand{\arraystretch}{1.20}
\begin{tabularx}{\textwidth}{@{}Ycccccc@{}}
\toprule
\headrow System & population & AUC@$N$ & $\sAUC$($1$) & decision acc. & GPU share & compute (MACs) \\
\midrule
Codec stage-1 (anytime) & 27k cell & $0.64$ & $0.80$ & $0.75\;[0.73,0.77]$ & $0\%$ & $10^{5}$ \\
Pixel CNN / chunk & 27k cell & $0.76$ & $0.86$ & $0.76\;[0.75,0.77]$ & $100\%$ & $1.8{\times}10^{10}$ \\
ReStraV (offline, full clip) & 27k cell & $0.95$ & n/a & $0.89\;[0.89,0.90]$ & $100\%$ & ${\sim}1.4{\times}10^{11}$ \\
\ourrow
\textbf{Ours (defer to pixel)} & 27k cell & n/a & n/a & $\mathbf{0.78}^{\dagger}\;[0.76,0.79]$ & $\mathbf{15\%}$ & $2.6{\times}10^{9}$ \\
\midrule
Codec stage-1 & subsample & $0.72$ & $0.88$ & $0.63^{\ddagger}\;[0.60,0.66]$ & $0\%$ & $10^{5}$ \\
VLM / clip (cost upper bnd) & subsample & $0.89$ & n/a & $0.80^{\ddagger}\;[0.78,0.82]$ & $100\%$ & ${\sim}10^{13}$ \\
VideoVeritas (zero-shot) & subsample & $0.83$ & n/a & $0.82^{\ddagger}\;[0.80,0.83]$ & $100\%$ & ${\sim}10^{13}$ \\
\ourrow
\textbf{Ours (defer to VLM)} & subsample & n/a & n/a & $\mathbf{0.72}^{\dagger\ddagger}\;[0.68,0.75]$ & $\mathbf{24\%}$ & $2.4{\times}10^{12}$ \\
\bottomrule
\end{tabularx}
\caption{Measured operating points (E2--E4) by population: the matched 27k cell (upper) and the
balanced $2{,}127$-clip subsample where the VLM was affordable (lower). Blocks are not comparable
(see text). AUC@$N$: running-max AUC at each clip's own full length. Figure~\ref{fig:anytime} stops at $16$ GOPs. Its endpoints $0.674$/$0.782$ (codec/pixel) sit consistently above these. Decision acc.: accuracy of the emitted decision. Brackets: clip-level 95\% bootstrap CIs. GPU share: fraction of clips reaching a
GPU stage (cascade rows: the deferral rate). The pixel cascade is a point of Figure~\ref{fig:pareto}.
Compute: expected MACs per scored chunk (per clip for the VLM, ReStraV, and defer-to-VLM rows; wall-clock in
the appendix). n/a: a cascade emits a
decision, not a score, and the VLM produces one score per clip. $^{\dagger}$accuracy at the
$\alpha{=}0.05$ gate. $^{\ddagger}$balanced accuracy; the VLM is zero-shot (AUC $[0.87,0.90]$).
VLM/ReStraV MACs are order-of-magnitude estimates, the only numbers not from the released code
path.}
\label{tab:results}
\end{table}

\noindent\textbf{E1 (positive control).} The harness must reproduce a known answer: the stage-1
detector's previously reported offline leave-one-generator-out AUC of $0.832$ \cite{vidaudit}. It does, exactly, on the released
feature table. Re-extracting every clip's features with our own pipeline gives $0.830$, and the
real-vs-real identity floor sits at $0.63$, where $0.5$ would be ideal. That floor frames E2's full-length $0.64$. Early prefixes clear it ($\sAUC(1)=0.80$). Streaming results are thus not a harness artifact.

\noindent\textbf{E2 (latency-accuracy frontier).} The pixel CNN is more accurate. Its AUC@$N$
is $0.76$ against the codec stage's $0.64$, and leads at every prefix
(Figure~\ref{fig:anytime}, Table~\ref{tab:results}). The codec stage answers on the other two axes. It is already informative after the first GOP ($\sAUC(1)$ $0.80$ vs $0.86$).
It costs five orders of magnitude less compute ($10^{5}$ vs $1.8{\times}10^{10}$ MACs per
chunk) and does not need any GPU. Its added streaming cost is the $\approx\!33$\,ms CPU parse-and-score. This
is the cheap, always-on front end.

\noindent\textbf{E3 (compute-accuracy frontier + deferral guarantee).} Sweeping the deferral width $w$
traces the measured frontier of Figure~\ref{fig:pareto}. Escalating $15\%$ of clips to the pixel stage lifts decision accuracy from $0.75$ (codec alone) to $0.78$ at $2.6{\times}10^{9}$ expected MACs. That is $7\times$ below scoring every chunk with the pixel CNN. Under clip-level pairing the gain is
$+0.027$ ($[+0.016,+0.038]$; McNemar $p=4{\times}10^{-7}$, $60$ vs $130$ flips). The compute
accounting is exact. The measured MACs match Proposition~\ref{prop:cascade}'s closed form to floating
point.
The frontier itself is monotone only up to the knee and then bends back to an accuracy of $0.69$
at $89\%$ deferral. As the window widens, the cascade escalates clips where stage~1 was already correct,
and the deferral condition fails in more than half of them. That is the conditional monotonicity of
Corollary~\ref{cor:budget}, observed with its negative regime. The knee is where we operate
(the $15\%$ row of Table~\ref{tab:results}).

\noindent\textbf{E4 (reasoning on demand).} We compare the two escalation choices, a pixel CNN and a
small VLM reasoning detector (Ivy-xDetector, Qwen2.5-VL-3B, run zero-shot), on a balanced $2{,}127$-clip
subsample. Both arms see the identical deferred set at a given budget, isolating the stage-2
choice. As a standalone, the VLM is the most accurate detector on this subsample. Its AUC is $0.89$ $[0.87,0.90]$ against the codec stage's $0.72$ on the same clips. It costs ${\sim}3{,}800\times$ the knee's expected compute (${\sim}10^{13}$ MACs, $16$\,s per clip). That cost is what the cascade rations. Where the choice matters is the deferred set. On the near-boundary clips the codec stage is at chance ($0.50$ balanced accuracy, by construction). The VLM classifies them far better than the pixel CNN, $0.80$ against $0.61$. The advantage holds at every budget measured. Escalating $24\%$ of clips reaches $0.72$ balanced accuracy with the VLM against $0.65$ with the pixel CNN ($+0.06$ $[+0.04,+0.09]$, McNemar $p=1{\times}10^{-11}$). At $55\%$, escalation reaches $0.79$ against $0.68$ ($+0.11$ $[+0.07,+0.15]$, $p=2{\times}10^{-19}$).

The caveat is the deferral budget. The codec stage is weak at the strict $\alpha{=}0.05$ gate on this subsample ($0.63$ balanced accuracy). The cascade must therefore escalate a quarter of clips to gain $+0.09$ over stage~1. It nears the VLM's standalone $0.80$ only at $55\%$ deferral ($0.79$). There it buys compute savings, not accuracy, over running the VLM everywhere.
Reasoning on demand is thus worth its cost in the regime the frontier makes explicit, not
at the ``few percent'' a stronger stage~1 would allow.

\noindent\textbf{E5 (how early can we decide).} The guarantees hold where they should. The
end-calibrated threshold keeps the stopping-time false-positive rate at $0.039$ for target
$\alpha{=}0.05$, and Proposition~\ref{prop:anytime}'s monotonicity holds for every clip. However, in the matched cell the anytime curve declines with prefix (Figure~\ref{fig:anytime}). The cause is the cell, not the signal. Generated clips are shorter, so their maxima freeze early. The real clips' maxima keep accruing false-alarm opportunities (Proposition~\ref{prop:anytime} under unequal lengths). On a long-form cell whose real and generated clips are matched in length, the curve is flat ($0.68$ across all prefixes). The stopping-time FPR stays controlled at $0.041$. There the first GOP is
as informative as the whole clip, so waiting gains nothing.

\noindent\textbf{E6 (cross-dataset).} We repeat the protocol on AIGVDBench (seven generators
plus a real source, $3{,}159$ clips). Leave-one-generator-out AUC is $0.62$ (bootstrap $[0.58,0.65]$), within the spread of GenVidBench's $0.64$. The anytime curve keeps the same early-informative, prefix-declining shape. One difference strengthens the result: motion vectors here come from each clip's source
codec, not our re-encode, so the signal survives a change of encoder as well. The gate does not transfer unchanged. Its
stopping-time FPR rises to $0.068$ under this shift (E7). The trend replicates, though one cross-dataset cell is limited evidence of external
validity.

\noindent\textbf{E7 (ablations).} Five ablations follow (tables in the appendix).

\emph{The gate} behaves as Proposition~\ref{prop:avi} predicts where its calibration null matches
the test null (stopping-time FPR $0.039$--$0.041$ at $\alpha{=}0.05$). It drifts above target when that null shifts ($0.068$ cross-dataset). Naive per-prefix recalibration matches it on the matched cell, but inflates the stopping-time rate by $1.7$--$2.8\times$ once the null shifts (Appendix E7-i). The end-calibrated threshold remains the better construction even where the guarantee is lost.
\emph{The chunk size} is a granularity knob. Eight frames is too short for the spectral features, while $16$ and $32$ tie at the full
prefix.

\emph{The feature set} behaves counterintuitively. The four global motion statistics generalize better across generators than the full 13-d feature. The nine spectral dimensions alone collapse out of distribution. The robust signal is thus coarse motion statistics, not fine spectral structure. \emph{The aggregator} costs accuracy. The running max is the weakest offline aggregate we measured (mean-pooling the same codec scores reaches $0.885$). But it is the only monotone one, and the guarantees require monotonicity (Appendix E7-v).

\emph{The motion-bias control} (the shortcut critique of \cite{motionbias}) runs
on all seven generators. Where stage~1 separates real from generated, the signal survives matching on motion magnitude and holds within motion quantiles (cogvideo: raw $0.90$, matched $0.91$, within-bin $0.81$). This is not the ``moves less'' artifact. On the hard generators stage~1 is at or near chance (raw $0.43$--$0.55$), so the control cannot inform. A residual shortcut cannot be excluded there.

\noindent\textbf{Reading the table.} Rows of Table~\ref{tab:results} are comparable only within a block. The subsample is easier (codec
$0.72$ against $0.64$ on the full cell), so comparing across blocks would overstate the VLM.
The escalation comparison that matters is inside the lower block. Deferring to the VLM reaches
$0.79$ against $0.68$ for the pixel CNN at the matched $55\%$ budget.

\section{Limitations}
\label{sec:limits}

The stream model assumes bitstream access and a codec whose motion field is informative. Transcoding is a calibration step, not free. Proposition~\ref{prop:avi}'s control rests on one end threshold calibrated on a held-out real sample. E7 shows it drifts once the test null departs from that sample. Proposition~\ref{prop:cascade} is conditional on stage~2 being
at least as accurate on the deferred set, which we test rather than assume. Motion-cue detectors can ride dataset motion biases
\cite{motionbias}, so we report the E7 control and leave bias-robust features as open work. Our scope is fully-synthetic AIGV, not face-swap calls \cite{activeprobe,realtimedf}. The VLM remains an
external dependency.

\section{Conclusion}
\label{sec:conclusion}

We recast AIGV detection as a streaming problem, with its own metrics and protocol. We then showed that the compressed bitstream supplies a cheap, always-on stage. That stage decides early and escalates only when uncertain. Two results make the trade-offs explicit. Temporal monotonicity turns latency
into a calibrated anytime curve, at a controlled false-alarm rate. The confidence-deferral
guarantee turns escalation cost into a budgeted frontier, monotone exactly where the deferral
condition holds. Detection then costs a codec parse, not a large-model forward pass, and any
monotone stage~1 inherits both guarantees.

\begin{availability}
Code, configurations, and evaluation splits are at
\url{https://github.com/KurbanIntelligenceLab/streamdet}. The released harness reproduces
every number in this paper except the VLM MAC estimate, which is analytic
(\cref{app:compute}). Datasets are public and cited in \cref{sec:exp}.
\end{availability}

\begin{conflicts}
The authors declare no competing interests.
\end{conflicts}

\bibliography{references}

\appendix
\section{Proofs}
\label{app:proofs}
The body states \cref{prop:anytime,prop:avi,prop:cascade} and \cref{cor:budget} and gives the idea of each proof in place.
This section gives the full arguments. Each result is restated first so it can be read without the
body in hand; the statements restate those in the body. \cref{prop:anytime}'s proof is a
single line and is given in the paper itself. We prove \cref{prop:avi,prop:cascade} and \cref{cor:budget} in full,
and isolate one fact used repeatedly, the stochastic monotonicity of the running maximum, as a
lemma.

\begin{lemma}[Stochastic monotonicity of the running maximum]
\label{lem:stochmono}
Let $M_t=\max_{i\le t}s_i$. Then $M_t$ is stochastically non-decreasing in $t$: for every threshold
$x$, $\,\Pr(M_t\ge x)$ is non-decreasing in $t$ under any fixed distribution of the scores.
Consequently, for a fixed level $\alpha$ the upper $(1-\alpha)$-quantile
$\tau_t=\inf\{x:\Pr(M_t\ge x)\le\alpha\}$ is non-decreasing in $t$.
\end{lemma}

\begin{proof}
$M_t=\max(M_{t-1},s_t)\ge M_{t-1}$ pointwise, so $\{M_{t-1}\ge x\}\subseteq\{M_t\ge x\}$ as events
and $\Pr(M_{t-1}\ge x)\le\Pr(M_t\ge x)$ for every $x$; this is first-order stochastic dominance
$M_t\succeq M_{t-1}$. Quantile monotonicity is immediate: since $x\mapsto\Pr(M_t\ge x)$ is
non-increasing and dominated by $x\mapsto\Pr(M_{t+1}\ge x)$, the infimum of
$\{x:\Pr(M_t\ge x)\le\alpha\}$ can only grow with $t$, i.e.\ $\tau_t\le\tau_{t+1}$.
\end{proof}

\begin{proposition*}[Restatement of \cref{prop:avi}]
Let $P_0$ be the real (null) distribution, let $M_t=\max_{i\le t}s_i$, and let the detector raise
``generated'' at the alarm time $\sigma=\inf\{t\le N: M_t\ge\tau\}$ (and ``real'' if no such $t$).
\begin{enumerate}\itemsep1pt
\item[(i)] \textbf{A single end-calibrated threshold is anytime-valid.} Calibrate the threshold
$\tau$ so that the null tail at the final prefix is controlled: $P_0(M_N\ge\tau)\le\alpha$. Such a
$\tau$ exists: the upper $(1-\alpha)$-quantile of $M_N$ when $M_N$ has no atom there (then with
equality), and in general any $\tau$ above it. Then the false-positive rate at the data-dependent alarm time equals
the final-prefix rate and is controlled, $P_0(\sigma\le N)=P_0(M_N\ge\tau)\le\alpha$, and
simultaneously $P_0(M_t\ge\tau)\le\alpha$ for every prefix $t\le N$.
\item[(ii)] \textbf{Per-prefix recalibration is not.} Choosing $\tau_t$ so that
$P_0(M_t\ge\tau_t)=\alpha$ at each prefix (an atomless null, for concreteness) controls the
per-prefix rate but not the alarm-time rate: $P_0(\exists\,t\le N: M_t\ge\tau_t)\ge\alpha$, with
strict inequality iff an early crossing the final test misses has positive probability,
$P_0(\exists\,t<N:M_t\ge\tau_t,\,M_N<\tau_N)>0$: the multiple-looks inflation.
\end{enumerate}
\end{proposition*}

\begin{proof}
By \cref{prop:anytime}, $M_t$ is non-decreasing in $t$, so $M_t\le M_N$ and
$\{\exists\,u\le t: M_u\ge\tau\}=\{M_t\ge\tau\}$. \emph{(i)} The alarm fires by prefix $t$ iff
$M_t\ge\tau$ and by the end iff $M_N\ge\tau$; hence $P_0(\sigma\le N)=P_0(M_N\ge\tau)\le\alpha$ by
the calibration of $\tau$, and $P_0(M_t\ge\tau)\le P_0(M_N\ge\tau)\le\alpha$ for every $t$.
\emph{(ii)} Since $\{M_N\ge\tau_N\}\subseteq\{\exists\,t: M_t\ge\tau_t\}$ and
$P_0(M_N\ge\tau_N)=\alpha$, the union has probability at least $\alpha$. By
Lemma~\ref{lem:stochmono} the marginal thresholds are ordered, $\tau_t\le\tau_N$ for $t\le N$, so on
the event $\{\tau_t\le M_t\le M_N<\tau_N\}$ an early boundary is crossed while the final one is not;
its probability $P_0(M_t\ge\tau_t,\,M_N<\tau_N)$ is a lower bound on the excess the union carries over
$\alpha$ (the full excess sums such events over the looks). The union strictly exceeds $\alpha$
exactly when some such miss event carries positive mass. Necessary for that: $\tau_t<\tau_N$ for
some $t$ (the null genuinely evolves) and null mass in $[\tau_t,\tau_N)$; not sufficient, since the
joint law decides. With perfectly dependent scores ($s_1=\dots=s_N$), for instance, every early
crossing is also a final crossing and the union is exactly $\alpha$.
\end{proof}

\begin{proposition*}[Restatement of \cref{prop:cascade}]
Let stage~1 produce score $s$ with per-GOP compute $C_1$, and let stage~2 have compute $C_2$ and
be invoked exactly when $s\in W$, both at constant unit costs; the cascade decides by stage~2 on
$\{s\in W\}$ and by stage~1's deterministic rule elsewhere. Writing $\mathrm{err}^{A}$ for the \emph{error mass} on a region
$A$ (the probability of being wrong \emph{and} in $A$), the expected per-clip compute and the
cascade error are
\[
\begin{aligned}
\mathbb{E}[C]&=\bar{t}\,C_1+\Pr(s\in W)\,C_2,\\[2pt]
\mathrm{err}_{\mathrm{casc}}&=\mathrm{err}_1-\bigl(\mathrm{err}_1^{W}-\mathrm{err}_2^{W}\bigr),
\end{aligned}
\]
where $\bar t$ is the expected number of GOPs processed and $\mathrm{err}_1$ is stage~1's error if
it decided everywhere.
\end{proposition*}

\begin{proof}
Compute is additive over GOPs processed and over the deferral event: stage~1 runs on every
processed GOP (expected count $\bar t$, cost $C_1$ each) and stage~2 runs iff $s\in W$ (cost
$C_2$), giving the stated expectation by linearity. For accuracy, the cascade equals stage~1 on
the confident set $\{s\notin W\}$ and stage~2 on the deferred set $\{s\in W\}$, so
$\mathrm{err}_{\mathrm{casc}}=\mathrm{err}_1^{\bar W}+\mathrm{err}_2^{W}$. Since
$\mathrm{err}_1=\mathrm{err}_1^{\bar W}+\mathrm{err}_1^{W}$, subtracting gives
$\mathrm{err}_{\mathrm{casc}}=\mathrm{err}_1-(\mathrm{err}_1^{W}-\mathrm{err}_2^{W})$; the
equivalence and strictness follow. The accounting idealizes stage-1 and stage-2 as fixed unit costs
$C_1,C_2$ and folds their early-exit behavior into the expected processed count $\bar t$;
input-dependent costs only rescale the compute axis and leave the accuracy statement untouched.
\end{proof}

\begin{corollary*}[Restatement of \cref{cor:budget}]
For an expected-compute budget $\mathbb{E}[C]=B$ with
$\bar t\,C_1\le B\le \bar t\,C_1+\Pr(s<\tau)\,C_2$ (gate-fired clips are never deferred), a width
$w$ with deferral probability $\Pr(s\in W)=(B-\bar t\,C_1)/C_2$ meets the budget exactly when the
score law is atomless below $\tau$; the sweep leaves $\bar t$ and $\mathrm{err}_1$ unchanged, since
deferral is decided only at the decision point. The achieved operating
point is $(\mathbb{E}[C],\,\mathrm{acc})=(\bar t\,C_1+\Pr(s{\in}W)\,C_2,\;1-\mathrm{err}_1+\Delta(W))$
with $\Delta(W)=\mathrm{err}_1^{W}-\mathrm{err}_2^{W}$. Sweeping the deferral width $w$
($W=W(w)=[\tau-w,\tau)$) traces this frontier, which is monotone, accuracy non-decreasing in the
budget, iff every marginal region added as $W$ widens satisfies the deferral condition
($\mathrm{err}_2\le\mathrm{err}_1$ there); where it fails the frontier bends back.
\end{corollary*}

\begin{proof}
\emph{Feasibility.} Setting $\mathbb{E}[C]=\bar t\,C_1+\Pr(s\in W)\,C_2$ equal to $B$ and solving
gives $\Pr(s\in W)=(B-\bar t\,C_1)/C_2$, which lies in $[0,\Pr(s<\tau)]$ exactly on the stated
budget range $\bar t\,C_1\le B\le\bar t\,C_1+\Pr(s<\tau)\,C_2$. Because $\Pr(s\in W(w))$ is non-decreasing and (for an atomless
score law) continuous in $w$, ranging from $0$ at $w=0$ up to $\Pr(s<\tau)$, every feasible deferral
probability is realized by some width $w$; the budget is met.

\emph{Accuracy along the frontier.} By \cref{prop:cascade}, $\mathrm{acc}(w)=1-\mathrm{err}_{\mathrm{casc}}
=1-\mathrm{err}_1+\Delta(W(w))$, so the fixed term $1-\mathrm{err}_1$ carries no $w$-dependence and
the entire frontier shape is $\Delta(W(w))$. Take $w'>w$ and write $W(w')=W(w)\,\dot\cup\,A$ with
$A=[\tau-w',\tau-w)$ the marginal region added, disjoint from $W(w)$. Error mass is additive over
disjoint regions, $\mathrm{err}_k^{W(w')}=\mathrm{err}_k^{W(w)}+\mathrm{err}_k^{A}$ for
$k\in\{1,2\}$, so
\[
\Delta(W(w'))-\Delta(W(w))=\mathrm{err}_1^{A}-\mathrm{err}_2^{A}.
\]
Widening the window also raises the budget, $B(w')>B(w)$, since $\Pr(s\in W)$ grows. Hence as the
budget increases across $A$ the achieved accuracy changes by exactly $\mathrm{err}_1^{A}-\mathrm{err}_2^{A}$,
the deferral gain on $A$. Therefore accuracy is non-decreasing in the budget iff
$\mathrm{err}_2^{A}\le\mathrm{err}_1^{A}$ for every marginal region $A$, i.e.\ $\mathrm{err}_2\le\mathrm{err}_1$
restricted to $A$; where instead $\mathrm{err}_2^{A}>\mathrm{err}_1^{A}$ on some added region, accuracy
strictly decreases there and the frontier bends back, escalation being spent where stage~1 was
already at least as accurate. In differential form, with the added sliver $[\tau-w-dw,\tau-w)$,
$\tfrac{d}{dw}\,\mathrm{acc}(w)$ has the sign of the local gain $(\mathrm{err}_1-\mathrm{err}_2)$ on
that sliver.
\end{proof}

\section{The streaming pipeline}
\label{app:pipeline}
\Cref{sec:method} specifies this system and \cref{alg:stream} states it
formally: the always-on CPU lane, the gate, and the rare escalation to a GPU stage~2.

\paragraph{Chunking versus encoder GOPs} The chunk is the detector's own fixed window ($16$
frames by default), not the encoder's segmentation. The distinction matters in the wild, where a
clip's GOP structure can be far coarser than the decision granularity one wants: our benchmark
clips carry a single I-frame each, one encoder GOP per clip. The detector therefore windows the
arriving frames itself; under the canonical closed-GOP-$16$ re-encode the two coincide, which is
why we say ``GOP'' for the chunk.

\section{Ablation detail (E7)}
\label{app:ablations}
Table~\ref{tab:ablations} gives the numbers behind the chunk-size and feature-set ablations
summarized in \cref{sec:exp}.

\begin{table}[h]
\centering
\setlength{\tabcolsep}{6pt}
\renewcommand{\arraystretch}{1.22}

\small
\setlength{\tabcolsep}{5pt}
\renewcommand{\arraystretch}{1.20}
\begin{tabularx}{\textwidth}{@{}lYcc@{}}
\toprule
\headrow
Ablation & Setting & AUC@$N$ & $\sAUC(1)$ \\
\midrule
\multirow{3}{*}{Chunk size} & $8$ frames & $0.53$ & $0.73$ \\
 & $16$ frames (default) & $0.64$ & $0.80$ \\
 & $32$ frames & $0.64$ & $\mathbf{0.84}$ \\
\midrule
\multirow{3}{*}{Feature set} & Tier-0 only (4 global motion stats) & $\mathbf{0.66}$ & n/a \\
 & Tier-1 only (9 spectral) & $0.31$ & n/a \\
 & All 13-d (default) & $0.64$ & $0.80$ \\
\bottomrule
\end{tabularx}
\caption{Chunk-size and feature-set ablations (matched 27k cell, leave-one-generator-out);
the adopted default in each group is marked \emph{(default)}; best per column in bold.
An $8$-frame window starves the spectral estimators; $32$ frames buys the best first-GOP score at a
coarser decision granularity. Tier-0 alone \emph{beating} the full feature, and Tier-1 alone falling
below chance out of distribution, together say the cross-generator signal lives in motion magnitude
rather than fine spectral structure. The spectral dimensions appear to fit generator-specific
detail that does not transfer.}
\label{tab:ablations}
\end{table}

\paragraph{Gate (E7-i)} Stopping-time false-positive rate at $\alpha{=}0.05$, end-calibrated
threshold versus the naive per-prefix recalibration foil, per cell: matched $0.039$ vs $0.039$;
long-form $0.041$ vs $0.069$; pixel $0.058$ vs $0.104$; AIGVDBench $0.068$ vs $0.191$. The
end-calibrated gate is at or near target on the two cells whose calibration null matches their test
null, and drifts above it under distribution shift, a regime \cref{prop:avi} does not
cover, since its guarantee is stated against the null it was calibrated on. The foil is worse on
every cell, and its inflation grows with the number of effective looks, exactly as
\cref{prop:avi}(ii) predicts.

\paragraph{Motion-bias control (E7-iv)} The control was run on all seven generators; per-fold
numbers below (raw AUC; AUC of mean motion magnitude alone; AUC after matching real and generated
on motion magnitude; motion-quantile within-bin AUC). On the separable generator (cogvideo)
matching does not degrade the score ($0.898$ raw vs $0.911$ matched) and the within-bin signal
($0.810$) far exceeds what motion magnitude alone provides ($0.663$), so the detector is not
reducible to the ``generated video moves less'' artifact of \cite{motionbias} there. On the hard
generators the control runs but cannot inform: stage~1 is at or near chance to begin with, so no
signal remains whose provenance the control could test.

\begin{table}[t]
\centering
\small
\setlength{\tabcolsep}{5pt}
\renewcommand{\arraystretch}{1.20}
\begin{tabularx}{\textwidth}{@{}Ycccc@{}}
\toprule
generator & raw & motion-only & matched & within-bin \\
\midrule
cogvideo & $0.898$ & $0.663$ & $0.911$ & $0.810$ \\
vc2      & $0.689$ & $0.522$ & $0.721$ & $0.686$ \\
ms       & $0.689$ & $0.393$ & $0.673$ & $0.680$ \\
svd      & $0.632$ & $0.625$ & $0.668$ & $0.626$ \\
mora     & $0.551$ & $0.621$ & $0.579$ & $0.558$ \\
pika     & $0.523$ & $0.211$ & $0.451$ & $0.482$ \\
musev    & $0.432$ & $0.316$ & $0.401$ & $0.492$ \\
\bottomrule
\end{tabularx}
\caption{Motion-bias control, per held-out generator. \emph{raw}: stage-1 AUC@$N$ as measured.
\emph{motion-only}: a control scored from coarse motion statistics alone. \emph{matched}: real and
generated clips matched on motion magnitude. \emph{within-bin}: scored inside motion quantiles.}
\label{tab:motionbias}
\end{table}

\paragraph{Aggregator (E7-v)} Why the running \emph{max}? Offline, at the clip's full length, it
is the weakest aggregate of the three we measured: fold-mean AUC $0.641$ (max) vs $0.885$ (mean
over GOPs) vs $0.850$ (last GOP) for the codec stage, and $0.759$ vs $0.869$ vs $0.863$ for the
pixel stage. But mean and last are not monotone in the prefix, so neither admits an anytime
decision rule: the monotonicity that \cref{prop:anytime} establishes for the max is lost, and with it
the single-threshold stopping-time guarantee of \cref{prop:avi} (a mean can rise and fall, so early
crossings are not nested and the multiple-looks inflation returns). The max is the price of anytime validity, and the deferral band is how the
cascade buys some of the gap back. The mean-pooled numbers also show the codec substrate itself
is rich: pooled offline it \emph{exceeds} the per-chunk pixel CNN ($0.885$ vs $0.869$).

\paragraph{Learned deferral rule} A logistic rule trained (within-test, 5-fold) to predict which
segments stage~2 would correct, versus the confidence band, as (band / learned) accuracy at matched
deferral budgets: $0.715/0.713$ at $2\%$, $0.718/0.717$ at $10\%$, $0.719/0.720$ at $15\%$, and
$0.717/0.723$ at $30\%$. The band is the better selector at the small budgets the system actually
operates at, and is overtaken only past ${\approx}15\%$, once the budget is large enough for a
learned rule to exploit structure beyond distance-to-threshold.

\section{Per-generator results}
\label{app:pergen}
The headline leave-one-generator-out numbers are means over seven held-out generators, and that mean
hides a wide spread that matters for reading the cascade results. For the codec stage on the matched
cell, per-generator AUC@$N$ is: cogvideo $0.90$, vc2 $0.70$, ms $0.70$, svd $0.65$, mora $0.56$,
pika $0.52$, musev $0.45$ (mean $0.64$, fold s.d.\ $0.14$). The stage is genuinely strong on one
generator, useful on three, and at or near chance on three. This is the spread that makes
escalation worth its compute: the deferred band is populated overwhelmingly by the hard generators,
which is why stage~2's accuracy on that band (E4) is the quantity that moves the frontier, not
stage~1's mean.

\section{Streaming-perception metric, adapted}
\label{app:metric}
Streaming perception evaluates a detector against the state of the world at the time its output
is emitted, folding latency into accuracy \cite{streamyolo}. We adapt this to a binary
authenticity decision: rather than matching to a moving ground truth, we evaluate the running
score $M_t$ against the (static) clip label at each prefix $t$, and report accuracy as a function
of the latency $\ell_t$ at which a decision is emitted. The latency-budgeted $\sAUC(B)$ of
\cref{def:metrics} is the accuracy reachable within budget $B$; the anytime curve is
its trace over budgets. This reduces to standard offline AUC at $t=N$ with unbounded budget,
which is the regime E1 checks.

\section{Anytime-valid inference and Ville's inequality}
\label{app:avi}
A test is \emph{anytime-valid} at level $\alpha$ if its false-positive probability is at most
$\alpha$ at every sample size and at any data-dependent stopping time \cite{savi,evalues}. The
standard route is a nonnegative supermartingale (an e-process) $E_t$ with $E_0\le 1$ under the
null; Ville's inequality gives $P_0(\sup_t E_t\ge 1/\alpha)\le\alpha$, so stopping the first time
$E_t\ge 1/\alpha$ controls the false-positive rate uniformly in time, the sequential
generalization of Wald's probability ratio test \cite{wald1945}. Confidence sequences are the dual
construction for estimation. \cref{prop:avi} is the special case in which the test
statistic is the running maximum $M_t$: because $M_t$ is monotone, the time-uniform bound follows
directly from a single quantile of $M_N$ without constructing a martingale, and the naive
per-prefix boundary is exactly the inadmissible ``continuous-monitoring'' rule these tools are
designed to replace. For a non-monotone streaming statistic (for example a windowed mean of
per-GOP scores), one recovers the same guarantee by instantiating $E_t$ as a test martingale and
applying Ville's inequality.

\cref{prop:avi} controls the false-positive rate at the stopping time. A constant false-positive
\emph{profile} across prefixes is a stronger demand: it needs a genuinely anytime-valid boundary
(a confidence sequence or e-process as above), which is more conservative than the marginal
per-prefix quantile that E7-i evaluates.

\section{Confidence-based deferral}
\label{app:defer}
The accuracy condition in \cref{prop:cascade} is the standard cascade-deferral
condition: a confident first stage should defer exactly the inputs on which a more capable second
stage is expected to be more accurate \cite{cascadedefer,earlyexitcascade}. Confidence-based
gating (defer when the stage-1 score lies in a band around the decision boundary) is a common,
well-behaved instance \cite{earlyexitcascade}; learned deferral rules can improve it and are
ablated in E7. We do not claim a new deferral method; we use the existing one and specialize its
compute and accuracy accounting to the streaming two-stage detector.

\section{Additional related work}
\label{app:related}
\paragraph{Provenance and watermarking} A complementary line labels synthetic media at the
source: signed provenance \cite{c2pa}, watermarks \cite{synthid}, encouraged by regulation
\cite{euaiact}. Both are cooperative-source solutions that transcoding strips and most
in-the-wild video never carried, so absence of a signal is not evidence of authenticity. Passive
detection remains the safety net; the codec signal we read is intrinsic to the very
re-compression that strips provenance.

\section{External-baseline configuration}
\label{app:baselines}
Two current external detectors anchor the compute-accuracy plane in \cref{tab:results}; both
were run under the audited protocol rather than quoted from their papers.

\paragraph{ReStraV (trained, matched protocol)} We port the perceptual-straightening recipe of
Intern\`o et al.\ verbatim: $24$ frames sampled over a $2$-second window centred on the clip, at
$224{\times}224$, embedded with a frozen DINOv2 ViT-S/14 (public \texttt{torch.hub} weights); the
embedding trajectory is summarised by successive step distances, turning angles, and their
fourth-moment statistics into the paper's $21$-d descriptor. The readout is the same
leave-one-generator-out $L_2$-regularised logistic regression as stage-1's --- identical folds,
seed, and calibration carve-out --- so its row in \cref{tab:results} differs from ours only in the feature,
not the protocol. Its MAC figure (${\sim}1.4{\times}10^{11}$ per clip) is the DINOv2-S forward at
$224^2$ over $24$ frames; the geometry summary and linear readout are negligible against it.

\paragraph{VideoVeritas (zero-shot)} We run the authors' released Qwen3-VL-8B checkpoint
(ModelScope weights) with the inference prompt reproduced verbatim from their repository, greedy
decoding, at most $1{,}024$ new tokens. $16$ frames are sampled uniformly across the clip at
native aspect ratio with the long side capped at $896$\,px ($32$ vision patches of $28$\,px;
uncapped $720$p-and-above sources exceed a $32$\,GB GPU in a single allocation). The soft
score is the softmax mass on the model's final real-vs-fake verdict token, falling back to the
parsed \texttt{<answer>} verdict when that token cannot be located; the hard verdict thresholds
this score at $0.5$. No fine-tuning, recalibration, or prompt search.

Scored on the same $2{,}127$-clip population as the Ivy-VLM arm, it reaches AUC $0.831$
$[0.815,0.845]$ and balanced accuracy $0.819$ $[0.804,0.833]$, at a median $22.9$\,s per clip.
Per-generator accuracy is sharply bimodal: cogvideo $0.92$, pika $0.95$, vc2 $0.94$, and the real
sources hd\_vg $0.96$ and vript $0.98$, against svd $0.15$, musev $0.38$ and ms $0.45$. Mean score
tracks accuracy almost exactly on every generator, so on the three hard generators the model is
not uncertain but confidently wrong.

\paragraph{Excluded: real-time face-deepfake detectors} These \cite{activeprobe,realtimedf} are
not run as comparison baselines. They authenticate faces manipulated in live calls, not
fully-synthetic scenes, so they target a different signal and would be evaluated on a different
test distribution (face crops rather than whole clips). They are complementary to streaming AIGV
detection, which is how we position them.

\section{Compute accounting and hardware}
\label{app:compute}
We report two cost axes. \emph{Compute} is multiply-accumulate operations (MACs) per chunk,
$\mathbb{E}[C]=\bar t\,C_1+\Pr(s\in W)\,C_2$ as in \cref{prop:cascade}, with $C_1$ the
per-GOP parse-and-score cost and $C_2$ a single stage-2 inference. Measured on the matched cell,
$C_1\approx10^{5}$ MACs (the codec parse plus a $13$-d linear score) and $C_2\approx1.8{\times}10^{10}$
MACs (a CLIP ViT-B/32 forward over four sampled frames); because $C_1\ll C_2$ and $\Pr(s\in W)$ is
small by design, the expected compute is dominated by the rare stage-2 calls, which is what the
compute-accuracy frontier (E3) quantifies. This axis is hardware-independent and is where the codec
stage's advantage is unambiguous. \emph{Wall-clock} is more implementation-dependent: the
per-chunk parse-and-score is $\approx\!33$\,ms on CPU and one stage-2 forward is $\approx\!28$\,ms on
an A100, but both are preceded by a frame decode ($\approx\!0.4$\,s per GOP for the codec stage's
all-frame decode, less for the four-frame pixel sample). That decode is an artifact of extracting
motion vectors through the decoder's side-data export in our reference implementation; a bitstream
motion-vector parser reads the same field without pixel reconstruction, and a deployed platform
already decodes the stream for playback and moderation, so we treat the decode as a shared cost and
report the $10^{5}$-MAC parse as the compute the codec stage \emph{adds}.

\paragraph{Hardware} Runs were distributed across two university clusters, heterogeneous by
design: one carries Intel Xeon Gold 6248 nodes ($2{\times}20$ cores, $2.50$\,GHz) with NVIDIA
V100-PCIE-32\,GB GPUs under Red Hat Enterprise Linux 8.10 (kernel 4.18.0-553); the other AMD EPYC
7713 nodes (64 cores) with NVIDIA A100-SXM4-40\,GB under SUSE Linux Enterprise 15 SP6 (kernel
6.4.0). Both wall-clock figures in this section were measured in one job on an A100 node, so the
$\approx\!33$\,ms parse-and-score timing is against the AMD EPYC 7713. One software environment is
shared by both clusters: Python 3.12.13, PyTorch 2.13.0+cu126, NumPy 2.5.1, scikit-learn 1.9.0,
pandas 3.0.3, and PyAV 18.0.0 (FFmpeg libavcodec 62.28.102, the decoder whose side-data export
provides the motion vectors). The workload nonetheless fits a single consumer GPU: stage~1 is
CPU-only, the stage-2 CLIP ViT-B/32 forward needs under $2$\,GB, and the $3$B VLM upgrade fits in
under $8$\,GB, so the core results run within a $12$\,GB budget. Stage~2 is invoked on only the
deferred fraction ($15\%$ at our operating point), so the GPU is lightly loaded.

\end{document}

%% file: affiliations.tex
\DeclareAffiliation{hbku}{%
  College of Science and Engineering, Hamad Bin Khalifa University, Doha, Qatar}

\DeclareAffiliation{tamu}{%
  Department of Computer and Electrical Engineering,
  Texas A\&M University, College Station, TX, USA}

\DeclareAffiliation{iub}{%
  Luddy School of Informatics, Computing, and Engineering,
  Indiana University Bloomington, Bloomington, IN, USA}
